\documentclass{article}
\usepackage[preprint]{log_2025}			

\usepackage{booktabs}						
\usepackage{multirow}						
\usepackage{amsfonts}						
\usepackage{graphicx}						
\usepackage{tikz}
\usepackage{placeins}

\usepackage[
     backend=biber,
     style=numeric-comp,
     backref=true,
     natbib=true]{biblatex}
\usepackage{amsmath}
\usepackage{amssymb}
\usepackage{amsthm}
\usepackage{thmtools}
\usepackage{hyperref}
\usepackage[capitalize]{cleveref}
\usepackage{adjustbox}
\usepackage{verbatim}
\usepackage{thm-restate}
\usepackage{tablefootnote}
\usepackage{svg}
\usepackage{float}

\usepackage[textcolor=black, backgroundcolor=blue!20, linecolor=blue]{todonotes}

\DeclareMathOperator{\map}{map}
\DeclareMathOperator{\conv}{Conv}
\DeclareMathOperator{\cconv}{\mathbf{Conv}}

\DeclareMathOperator{\Z}{\mathbb{Z}}
\DeclareMathOperator{\C}{\mathbb{C}}
\DeclareMathOperator{\R}{\mathbb{R}}

\newtheorem{theorem}{Theorem}

\newtheorem{proposition}[theorem]{Proposition}

\makeatletter
\newcommand*\bigcdot{\mathpalette\bigcdot@{.5}}
\newcommand*\bigcdot@[2]{\mathbin{\vcenter{\hbox{\scalebox{#2}{$\m@th#1\bullet$}}}}}
\makeatother

\title[Rotation Equivariant Deep Learning for Contours]{RotaTouille: Rotation Equivariant Deep Learning for Contours}

\author[O. H. Gardaa et al.]{%
Odin Hoff Gardaa\\
University of Bergen\\
\email{odin.garda@uib.no}\And
Nello Blaser\\
University of Bergen\\
\email{nello.blaser@uib.no}
}

\begin{document}

\maketitle

\begin{abstract}
Contours or closed planar curves are common in many domains. For example, they appear as object boundaries in computer vision, isolines in meteorology, and the orbits of rotating machinery. In many cases when learning from contour data, planar rotations of the input will result in correspondingly rotated outputs. It is therefore desirable that deep learning models be rotationally equivariant. In addition, contours are typically represented as an ordered sequence of edge points, where the choice of starting point is arbitrary. It is therefore also desirable for deep learning methods to be equivariant under cyclic shifts. We present RotaTouille, a deep learning framework for learning from contour data that achieves both rotation and cyclic shift equivariance through complex-valued circular convolution. We further introduce and characterize equivariant non-linearities, coarsening layers, and global pooling layers to obtain invariant representations for downstream tasks. Finally, we demonstrate the effectiveness of RotaTouille through experiments in shape classification, reconstruction, and contour regression.
\end{abstract}

\section{Introduction}

Equivariance and invariance are the central concepts in geometric deep learning. Designing architectures that respect certain symmetries, often defined in terms of group actions, allows us to incorporate prior geometric knowledge about the data into the learning process. This can lead to models that generalize better, require less data, and are more efficient by reducing the effective hypothesis space. Equivariance is especially useful when the task requires the model to be sensitive to transformations of the input, such as translation, rotation, or permutation, while still producing consistent and meaningful output. Invariance, on the other hand, is desirable when the output should remain unchanged under transformations. Convolutional neural networks (CNNs) illustrate both concepts in image analysis. Convolutions are translation-equivariant, so shifting an input image shifts the feature maps accordingly, which is useful for segmentation, while applying global pooling after the convolutional layers achieves invariance, allowing classification to be insensitive to the object's position. Graph neural networks (GNNs) offer another example: they are often designed to be invariant or equivariant under permutations of node orderings (graph isomorphisms). A comprehensive overview of methods and concepts in geometric deep learning can be found in \cite{bronstein2021geometric} and \cite{gerken2023geometric}. 

\begin{figure}[ht!]
  \centering
  \includesvg[width=0.8\textwidth,inkscapelatex=true]{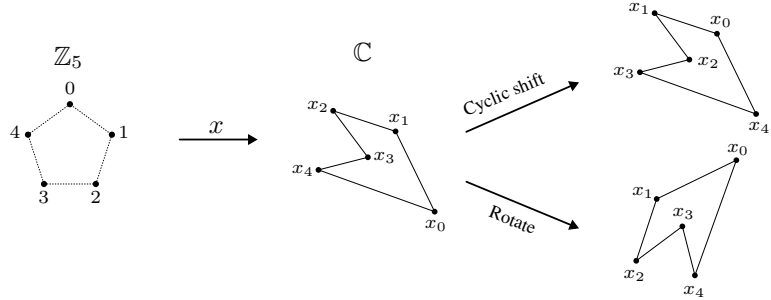}
  \caption{An illustration showing a contour $x\colon\Z_n\to\C^k$ where $n=5$ and $k=1$, and how a cyclic shift by one and a rotation by $\pi/2$ radians changes the image of $x$ in $\C$. For readability, we write $x_i=x(i)$. For $k>1$, one can think of the image of $x$ as a stack of contours, one for each copy of $\C$.}
  \label{fig:illustrated_contour_and_actions}
\end{figure}

\looseness=-1 In this paper, we focus on what we will refer to as \emph{contours}. Contours are complex-valued signals on finite cyclic groups, i.e., functions $\Z_n\to\C^k$, where $\Z_n$ denotes the cyclic group of order $n$, and $\C^k$ is the $k$-dimensional complex vector space consisting of $k$-tuples of complex numbers. See \cref{fig:illustrated_contour_and_actions} for an illustration of a simple contour. Contours can serve as sparse representations of object boundaries, but they are not limited to simple closed shapes. Contours can also accommodate more general signals, including self-crossing curves and multichannel contours. For example, contours occur naturally as cell shapes in cell morphology \cite{burgess2024orientation,bleile2023persistence,migicovsky2019}, and as orbit plots in the vibrational analysis of rotating machinery \cite{jablon2021diagnosis, jeong2016rotating, caponetto2019deep}. Our method uses complex-valued neural networks, which employ complex-valued weights and are often employed in applications such as radar imaging \cite{gao2018enhanced, scarnati2021complex}, MRI fingerprinting and reconstruction \cite{virtue2017better, cole2021analysis}, and other areas where data are naturally represented in the complex domain. See \cite{lee2022survey, bassey2021survey} for a detailed overview of complex-valued neural networks and their applications. Complex-valued CNNs have also been applied to images, showing self-regularizing effects \cite{guberman2016complex}.

\subsection{Contributions}
\looseness=-1 The group $G_n=\Z_n\times S^1$ acts on a contour $x\colon\Z_n\to\C^k$ by combining two operations: the cyclic group $\Z_n$ acts by cyclically \emph{shifting} the starting point of the contour, and $S^1$ acts by \emph{rotating} the image of $x$ about the origin in each copy of the complex plane. These actions are illustrated in \cref{fig:illustrated_contour_and_actions}. We propose rotation- and cyclic shift-equivariant (as well as invariant) layers for deep learning on contours, leveraging complex-valued convolutions over the cyclic group (circular convolution). Convolution is well known for its cyclic shift-equivariance, and working in the complex domain also ensures equivariance to planar rotations. Our framework, \textbf{RotaTouille}, comprises the following main layer types:


\begin{itemize}
	\item \textbf{Convolution layers.} Linear equivariant layers based on complex circular convolutions.
	\item \textbf{Activation functions.} Equivariant non-linear functions applied element-wise allowing the network to learn more complex functions.
	\item \textbf{Coarsening layers.} Equivariant local pooling layers downsampling the signal by coarsening the domain.
	\item \textbf{Invariant layers.} Global pooling layers producing real-valued invariant embeddings.
\end{itemize}
This layer taxonomy aligns well with the Geometric Deep Learning Blueprint proposed in \cite[p.~29]{bronstein2021geometric}. Throughout the paper, we show that the proposed layers are indeed equivariant and provide a classification of all possible equivariant non-linear activation functions. We evaluate RotaTouille in different tasks, including shape classification, shape reconstruction, and node-level curvature regression. Furthermore, we compare the effect of various design choices in an ablation study. Our implementation, including all code required to reproduce the experiments, is publicly available.\footnote{\url{https://github.com/odinhg/rotation-equivariant-contour-learning}}

\subsection{Related Work}

Shape analysis is a central topic in computer vision and machine learning, with early approaches often based on hand-crafted descriptors computed from contours. Examples include Curvature Scale Space (CSS) representations \cite{mokhtarian1997efficient, abbasi1999curvature} and generalized CSS (GCSS) \cite{benkhlifa2017novel}, which were later used in the DeepGCSS neural network classifier \cite{Mziou-Sallami2023}. The use of neural networks for contour data appeared already in \cite{gupta1990neural}, where the authors used fully connected neural networks to classify contours. Other classical works combined contour fragments with skeleton-based features to improve shape recognition \cite{shen2014shape, shen2016shape, shen2018bag}. The shape context (SC) descriptors \cite{belongie2000shape}, based on log-polar histograms, have also been shown to be effective in capturing the local geometric structure in a rotation-invariant way.

More recently, deep learning methods have become more popular for contour analysis. ContourCNN \cite{droby2021contourcnn} models planar contours using real-valued circular convolutions on point sequences, with a custom pooling strategy to discard shape-redundant points. A similar convolutional approach was proposed in \cite{mhedhbi2022new}. Two-dimensional CNNs have also been applied to contour images \cite{jeong2016rotating,atabay2016binary,caponetto2019deep} and transformer-based architectures have been adapted for shape data using SC descriptors in \cite{Jia2023shapetransformer}. None of these deep learning-based methods is intrinsically equivariant or invariant to rotations and they instead rely on data augmentation or manual feature extraction of invariant features. ShapeEmbed was recently introduced in \cite{romero2025shapeembed} as a self-supervised variational autoencoder (VAE) that takes the pairwise distance matrix of contour points as input. It produces latent representations that are invariant to a fixed set of symmetries, including rotations and cyclic shifts.

Several works have explored learning rotation equivariant or invariant representations more directly, particularly for images and point clouds. The O2-VAE model \cite{burgess2024orientation} applies steerable CNNs \cite{cohen2016steerable} to cell images, encoding them into a latent space that is invariant to planar rotations. Although effective in capturing texture- and intensity-based features, such approaches operate on pixel grids rather than directly on contour data. Group-equivariant CNNs (G-CNNs) \cite{cohen2016group} provide a general framework for building networks equivariant to discrete symmetry groups, though they have been applied mostly to images. Closer in spirit to our approach are quaternion neural networks for 3D point clouds \cite{shen20203d}, where rotation-equivariant layers are constructed using the action of unit quaternions on $\R^3$. Topological methods have also been applied to contour data. Persistent homology has been used to extract stable shape features for morphological classification and comparison \cite{bleile2023persistence, migicovsky2019}.

\section{Methodology}

We begin by defining contours and the action of the group $\Z_n \times S^1$ on them, along with the concepts of equivariant and invariant maps. Subsequently, we introduce the various equivariant and invariant layers that form RotaTouille. For completeness, \emph{groups} and \emph{group actions} are defined in \cref{sec:groups_and_actions}.

\subsection{Contours}

For an integer $n>0$, we let $\Z_n=\{0, 1, \ldots, n-1\}$ denote the \emph{cyclic group of order $n$} under addition modulo $n$. A \emph{contour} is a function $x\colon \Z_n\to \C^k$ where $k>0$. The set of all contours $X_n^k = \map(\Z_n, \C^k)$ is a vector space over $\C$ with addition and scalar multiplication defined pointwise. For any function $f\colon S\to\C^k$ where $S$ is a set, we use the shorthand notation $f_j=\pi_j\circ f$ where $\pi_j\colon \C^k\to\C$ is the projection onto the $j$-th coordinate. 

\subsection{Group Actions and Equivariant Maps}

The \emph{circle group} $S^1\subseteq\C$ consists of all unit complex numbers and the group product is given as complex multiplication. We consider the product group $G_n=\Z_n\times S^1$ and let it act on contours as follows: for $(l,w)\in G_n$ and a contour $x\in X_n^k$ we define $(l,w)\bigcdot x$ to be the contour mapping $q\mapsto wx(q-l)$. Here, $w$ acts by scalar multiplication in $\C^k$ and $q-l$ is computed modulo $n$. We show that this does indeed define a (left) group action in \cref{prop:group_action_well_defined}.


A function $f\colon X_n^k\to X_m^k$ where $m$ is a divisor of $n$ is called \emph{equivariant} if $f((l,w)\bigcdot x)=(l,w)\bigcdot f(x)$ for all $x\in X_n^k$ and all $(l,w)\in G_n$. That is, the function commutes with our group action on contours. A function $f\colon X_n^k\to Y$ where $Y$ is a set is called \emph{invariant} if $f((l,w)\bigcdot x)=f(x)$ for all $x\in X_n^k$ and all $(l,w)\in G_n$. The concepts of invariance and equivariance are more general than described here and can be defined for any group action, including group representations (also known as linear actions). For a broader exposition, see \cite[Chapter~3.1]{bronstein2021geometric} or \cite[Chapter~2.1]{weiler2023EquivariantAndCoordinateIndependentCNNs}.

\subsection{Circular Convolution as Linear Equivariant Layer}
If a map $T\colon X_n^1\to X_n^1$ satisfies $T(wx)=wT(x)$ and $T(x+y)=T(x)+T(y)$ for all $w\in S^1$ and $x,y\in X_n^1$, then $T$ is automatically $\C$-linear. This is a consequence of the fact that every complex number can be written as a (nonunique) finite sum of unit complex numbers. Now, suppose $T$ also commutes with cyclic shifts, so $T$ is any additive map that is equivariant with respect to our group action: $T((l,w)\bigcdot x)=(l,w)\bigcdot T(x)$. Then it is well known that $T$ is necessarily the circular convolution operator. For completeness, we include a proof of this in \cref{prop:linear_shift_eq_map_is_conv}.

The \emph{circular convolution} $\ast\colon X_n^1\times X_n^1\to X_n^1$ is defined by letting $(y\ast x)(q)=\sum_{j=0}^{n-1}y(j)x(q-j)$ for $x,y\in X_n^1$ and $q\in\Z_n$. In convolutional neural networks, we tend to work with kernels smaller than the signal. If $y\in X_m^1$ for some $m\leq n$, we extend $y$ with zeros to $X_n^1$ before convolving with $x$. That is, we let $y(q)=0$ whenever $q\geq m$. We also want our equivariant layers to handle multi-channel signals, both as inputs and as hidden representations. By convention, we integrate information across the input channels by summing. Fixing $\phi\in X_m^k$, we define the \emph{circular convolution operator} $\conv_\phi\colon X_n^k\to X_n^1$ as 
\[
	\conv_\phi(x) = \sum_{j=1}^k\phi_j\ast x_j
\]
for $x\in X_n^k$. We refer to $\phi$ as a \emph{filter (or kernel)} that is typically learned during training, and $m$ as the \emph{kernel size}. Note that we do not use an additive bias term as this would break rotational equivariance. In practice, we often have multiple filters in each convolutional layer. Let $\Phi=(\phi^1,\ldots,\phi^{k'})$ be a collection of $k'$ filters, often referred to as a \emph{filter bank}. The \emph{convolutional layer} denoted $\cconv_\Phi\colon X_n^k\to X_n^{k'}$ is defined coordinate-wise by letting $\cconv_\Phi(x)_j = \conv_{\phi^j}(x)$ for all $j=1,\ldots,k'$. We show that the convolutional layer satisfies the equivariance property. The proof that this satisfies equivariance is provided in \cref{sec:proofs}.

\begin{restatable}{proposition}{convequivariant}
The convolutional layer $\cconv_\Phi\colon X_n^k\to X_n^{k'}$ is equivariant, that is, for all $(l,w)\in G_n$ and $x\in X_n^k$, we have $(l,w)\bigcdot\cconv_\Phi(x)=\cconv_\Phi((l,w)\bigcdot x)$.
\end{restatable}

\subsection{Non-linear Activation Functions}\label{sec:nonlinearities}

To learn more complex, non-linear functions on the space of contours, we need to introduce non-linear activation functions between the linear layers. We also want these activation functions to be equivariant. Any function $a\colon\C\to\C$ can be extended to a function $X_n^k\to X_n^k$ by point-wise application in each coordinate, and this function is equivariant if and only if $a(wz)=wa(z)$ for all $z\in\C$ and $w\in S^1$. We classify such functions in the following proposition proven in \cref{sec:proofs}:

\begin{restatable}{proposition}{nonlinearequivariant}
A function $a\colon\C\to\C$ satisfies the equivariance condition $a(wz)=wa(z)$ for all $z\in\C$ and $w\in S^1$ if and only if there exists a function $g\colon[0,\infty)\to\C$ such that $a(z)=g(\vert z\vert)z$ for all $z\in\C$.
\end{restatable}

Functions on this form already appear in the existing literature on complex-valued neural networks and we list some of them in \cref{tab:equivariant_activations}.
\begin{table}[ht]
    \small
	\centering
	\caption{Examples of equivariant activation functions for contours.}
	\label{tab:equivariant_activations}
	\begin{tabular}{ll}
		\toprule
		\textbf{Activation Function} & \textbf{Choice of $g(r)$} \\
		\midrule
		Siglog \cite{virtue2017better} & $(r+1)^{-1}$\\
		ModReLU\tablefootnote{The ModReLU activation function has a learnable bias parameter $b\in\R$.} \cite{arjovsky2016unitary,worrall2017harmonic} & $\operatorname{ReLU}(r+b)r^{-1}$ \\ 
		Amplitude-Phase \cite{Hirose2012} & $\tanh(r)r^{-1}$ \\ 
		\bottomrule
	\end{tabular}
\end{table}

\subsection{Coarsening Functions}

A \emph{coarsening (or local pooling) function} is an equivariant function $P\colon X_n^k\to X_m^k$ with $m<n$. In this section, we suppose that $n=mp$ for some integer $p>1$ and think of $p$ as the coarsening factor. We fix two positive integers $n_s$ and $n_d$ called \emph{stride} and \emph{dilation}, respectively, and let $\bigoplus\colon \C^p\to \C$ be any function such that $\bigoplus(wz)=w\bigoplus(z)$ for all $w\in S^1$ and $z\in\C^p$. For convenience, we write $\bigoplus_{j=0}^{p-1}z_j$ instead of $\bigoplus(z_0,\ldots,z_{p-1})$. Examples include magnitude-based argmax $\bigoplus_{j=0}^{p-1}z_j = \arg\max_{j=0,\ldots,p-1}\vert z_j\vert$, and the mean $\bigoplus_{j=0}^{p-1}z_j = p^{-1}\sum_{j=0}^{p-1} z_j$ for $z\in\C^p$. For a contour $x\in X_n^1$ we define the coarsening function $P_\oplus\colon X_n^1\to X_m^1$ by letting

\[
P_\oplus(x)(q)=\bigoplus_{j=0}^{p-1}x(qn_s+n_dj)
\]

for all $q\in\Z_m$ and extend this to a function $P_\oplus\colon X_n^k\to X_m^k$ coordinate-wise. We refer to the case with $n_s=1$ and $n_d=m$ as \emph{coset pooling}. That is, we pool over the cosets $q+\langle m\rangle=\{q,q+m,\ldots,q+(p-1)m\}$ for $q\in\Z_m$ and the function $P_\oplus$ is equivariant as we have 

\[
P_\oplus((l,w)\bigcdot x)(q)=\bigoplus_{j=0}^{p-1}wx((q+mj)-l)=w\bigoplus_{j=0}^{p-1}x((q-l)+mj)=((l,w)\bigcdot P_\oplus(x))(q).
\]

For $n_s=p$ and $n_d=1$, we get \emph{strided pooling}, which is the most common type of pooling operation in CNNs. However, strided pooling is only approximately equivariant in the sense that cyclically shifting the input signal by $lp$, cyclically shifts the output signal by $l$:

\[
	P_\oplus((lp,w)\bigcdot x)(q)=\bigoplus_{j=0}^{p-1}wx((qp+j)-lp)=w\bigoplus_{j=0}^{p-1}x((q-l)p+j)=((l,w)\bigcdot P_\oplus(x))(q).
\]

In other words, strided pooling is equivariant with respect to the subgroup $p\Z_n\times S^1$ of $G_n$, but is true to the assumption that points close in the domain $\Z_n$ tend to be mapped to close points in $\C^k$. Coset pooling, on the other hand, is truly equivariant, but does not aggregate locally.

\subsection{Invariant Feature Extraction}

A \emph{global pooling function} is an invariant map $A\colon X_n^k\to Y$ where $Y$ is some set. We will only consider $Y=\R^k$. The idea then is that $A$ aggregates channel-wise information into a real-valued contour embedding that can be used in downstream tasks. In the implementation, we use a combination of the mean and maximum of absolute values similar to what is done in \cite{shen2018bag}. Specifically, we define the global pooling function $A\colon X_n^k\to\R^k$ by setting the $i$-th component of $A(x)$ to be
\begin{equation}\label{eq:global_pooling_learnable} 
	\alpha n^{-1}\sum_{j=0}^{n-1}\vert x_i(j)\vert+(1-\alpha)\max_{j=0,\ldots,n-1}\vert x_i(j)\vert
\end{equation}

for $x\in X_n^k$ and where $\alpha\in[0,1]$ is a learned parameter.


\section{Experiments}

We perform various experiments to assess the feasibility of RotaTouille. This section outlines the details of preprocessing and implementation, introduces the datasets, and presents the results.

\subsection{Preprocessing and Implementation Details}

\paragraph{Data preprocessing} For image data, we extract contours through a three-step process:
\begin{enumerate}
\item\textbf{Binarization.} For grayscale images, we convert the images to binary images by applying thresholding, for example, by using Otsu's method \cite{otsu1975threshold}.
\item\textbf{Contour extraction.} We use the OpenCV library \cite{opencv_library} to extract contours from the binary images, selecting the one with the largest area in the case of multiple contours.
\item\textbf{Equidistant resampling.} We resample the contour to a fixed number of points, equidistantly with respect to the Euclidean distance.
\end{enumerate}

See \cref{fig:contour_processing} for an illustration of the image-to-contour conversion process.

\begin{figure}[ht]
	\centering
	\includesvg[width=0.5\textwidth]{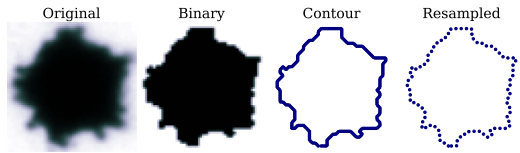}
	\caption{Image-to-contour conversion process. The input image is binarized, contours are extracted, and then resampled to a fixed length with equidistant points.} 
	\label{fig:contour_processing}
\end{figure}

Before passing the contours to the model, they are centered at their mean and rescaled by dividing by the standard deviation of the magnitudes, as this showed to improve convergence of the complex-valued models. The image datasets are standardized using statistics computed on the training dataset.

\paragraph{Multi-scale invariant features.}\looseness=-1 Choosing the optimal kernel size in 1D CNNs is not straightforward. One approach is to build an ensemble of models with different kernel sizes to capture features at multiple scales. While effective, this increases the number of parameters and computational cost. In our classification experiments, extracting invariant features at different network depths improved performance without adding learnable parameters. \Cref{fig:multi_scale_invariant_features} illustrates the structure of a convolutional block.

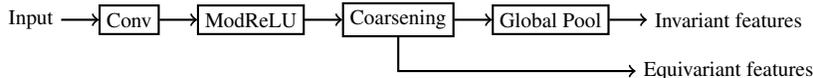
\begin{figure}[ht]
\centering
\begin{tikzpicture}[scale=1.0, every node/.style={scale=0.8}, thick, node distance={0.25cm and 0.5cm}]

\node (input) {Input};
\node[rectangle, draw, align=center, right=of input] (conv) {Conv};
\node[rectangle, draw, align=center, right=of conv] (modrelu) {ModReLU};
\node[rectangle, draw, align=center, right=of modrelu] (coarse) {Coarsening};
\node[rectangle, draw, align=center, right= of coarse] (pool) {Global Pool};

\node[right=of pool] (invfeat) {Invariant features};
\node[below=of invfeat] (eqfeat) {Equivariant features};

\draw[->] (input) -- (conv);
\draw[->] (conv) -- (modrelu);
\draw[->] (modrelu) -- (coarse);

\draw[->] (coarse) -- (pool);
\draw[->] (coarse.south) |- (eqfeat.west);

\draw[->] (pool) -- (invfeat);

\end{tikzpicture}

\caption{Multi-scale invariant feature extraction. After each convolutional block, a global pooling operation produces an invariant feature vector at that scale. Equivariant features are passed to the next block, and the final feature vector concatenates the invariant vectors from all depths.}\label{fig:multi_scale_invariant_features}

\end{figure}

\paragraph{Contour re-centering} Before every convolutional layer and global pooling layer, we re-center the contour channel-wise as this showed improvements in training stability and performance of our method. This re-centering is an equivariant operation.

\paragraph{Model architectures}
For classification and regression, we use the ModReLU activation with strided local pooling and global pooling layers that aggregate via a learnable combination of mean and absolute-value maximum. These choices were guided by preliminary experiments on related datasets, which indicated robustness. The classification model has seven convolutional layers followed by a fully connected head, while the regression model uses four convolutional layers with a linear head. For the \texttt{RotatedMNIST} classification task, radial histogram features are concatenated with the invariant features from the global pooling layers before the classification head. The contour autoencoder has five convolutional layers in both encoder and decoder, including pooling and unpooling layers. Full architecture details are provided in \cref{sec:model_details}.

\paragraph{Training, model selection and evaluation.}\looseness=-1 We use the Adam optimizer in all experiments and use $10\%$ of the training data for validation to choose the best model for evaluation on the separate test data. For the classification tasks, the test data is randomly rotated to test the robustness of the candidate models. In this case, we also apply random rotation to augment the training data for the non-invariant baseline models to make the comparison fair. Each experiment is repeated ten times with different seeds, and we report mean test scores together with standard deviations. For more details on hyperparameter values for the different model and dataset combinations, see \cref{tab:training_details} in the appendix.

\subsection{Datasets}

Now, we describe the five datasets used in our experiments. We have three datasets for shape classification, one for shape reconstruction, and one for node-level regression.

\paragraph{Fashion MNIST} The Fashion MNIST dataset \cite{xiao2017fashionmnist} contains $60{,}000$ training and $10{,}000$ test grayscale images of clothing items, each of size $28\times28$ pixels.  We convert these images to contours by applying the image to contour conversion process described earlier. For the baseline CNN, we create two versions of the dataset: one with filled contours and one with unfilled contours. The filled contours are binary images where the contour is filled in, while the unfilled contours are binary images where only the contour is drawn. Hence, we discard all texture information in the images and only use the shape of the clothing items. We refer to this dataset as \texttt{FashionMNIST}.

\paragraph{ModelNet} We create a multi-channel dataset based on ModelNet \cite{wu20153d}, which contains 3D CAD models from various object categories. We select the classes \texttt{bottle}, \texttt{bowl}, \texttt{cone}, and \texttt{cup}, chosen for their tendency to form a single connected component in cross-section. For each model, we uniformly sample a point cloud and divide it into four disjoint volumes along the second axis. From each volume, we generate a $64 \times 64$ binary image by projecting points onto an orthogonal plane, then extract contours from the four image channels. Each sample thus consists of a stack of four contour slices. The final \texttt{ModelNet} dataset contains 644 training examples and 160 test examples.

\paragraph{Rotated MNIST}\looseness=-1 The rotated MNIST dataset \cite{larochelle2007empirical} is a modified MNIST \cite{lecun2002gradient} where each image is randomly rotated by an angle uniformly sampled from $[0, 2\pi)$. It contains $12{,}000$ training and $50{,}000$ test examples of size $28\times28$ pixels, grayscale, with handwritten digits $0$–$9$. We convert the images to contours using the process described above and include a simple rotation-invariant texture feature based on the radial histogram (RH). The RH captures the distribution of pixel intensities by dividing the image into $14$ radial bins and counting pixels in each bin. We refer to this dataset as \texttt{RotatedMNIST}.

\paragraph{Cell Shapes}\looseness=-1 We use a subsample of the Profiling Cell Shape and Texture (PCST) benchmark dataset introduced \cite{burgess2024orientation} with $1000$ shapes from each of the $9$ categories corresponding to different combinations of eccentricity and boundary randomness. We extract contours and refer to this dataset as \texttt{PCST}.

\paragraph{Curvature contours}  We construct a synthetic dataset for curvature regression. Given a continuous contour $\gamma\colon[0,2\pi)\to\C$ where $\gamma(t) = x(t) + iy(t)$ with $x$ and $y$ twice differentiable, the \emph{curvature $\kappa$} is defined in terms of the first and second derivatives as
\[
\kappa = \frac{\left| x'y'' - y'x''\right|}{\left( x'^{\,2} + y'^{\,2}\right)^{3/2}}.
\]
The curvature measures the local deviation of the curve from a straight line. The \texttt{Curvature} dataset is generated as follows: First, we sample a number of modes $m$ uniformly from $\{2, 3, 4, 5\}$. Then, for each mode $k=1,\ldots,m$, we sample four coefficients $a_k^x, b_k^x, a_k^y, b_k^y$ independently from the uniform distribution on $[-1, 1]$, and define the contour $\gamma=x+iy$ where 
\[
	x(t) = \sum_{k=1}^{m} \left( a_k^x \cos(k t) + b_k^x \sin(k t) \right)\quad\text{and}\quad y(t) = \sum_{k=1}^{m} \left( a_k^y \cos(k t) + b_k^y \sin(k t) \right).
\]

This construction ensures periodicity and smoothness while introducing controlled geometric variability through the random coefficients and number of modes. Lastly, we sample $100$ points equidistant in arc length, and use the curvature in these points as the ground truth. See \cref{fig:curvature_contours_examples} for examples of generated curves. We discard contours with maximum curvature greater than $1000$ to avoid extreme values. We generate $2000$ contours for training and $1000$ contours for test data.

\begin{figure}[ht]
	\centering
	\includesvg[width=0.9\textwidth]{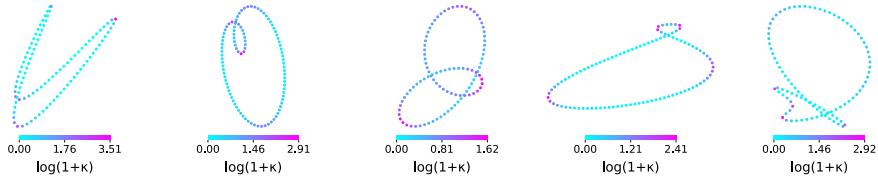}
	\caption{Five example contours from the \texttt{Curvature} dataset with log-curvature values colored.} 
	\label{fig:curvature_contours_examples}
\end{figure}

\subsection{Results}

\paragraph{Shape classification} \looseness=-1 We evaluate RotaTouille on the classification datasets \texttt{FashionMNIST}, \texttt{ModelNet}, and \texttt{RotatedMNIST}. The cross-entropy loss function is used for training and performance is measured using test accuracy (or test error). As a first proof of concept, we compare RotaTouille with the baseline models on \texttt{FashionMNIST}. The baseline CNN uses contour images, and we test both filled and unfilled contours. We also include a baseline graph neural network based on the graph convolutional network (GCN) introduced in \cite{kipf2016semi}. The graph is a fixed cycle graph with the Cartesian coordinates of the contour points as the node features. As an additional baseline, we implement the ContourCNN method using the optimal configuration reported in \cite{droby2021contourcnn}. To evaluate RotaTouille on multichannel data, we compare it with the baseline models on \texttt{ModelNet}. The results are listed in \cref{tab:fashion_mnist_results}, where RotaTouille slightly outperforms the baseline models on both datasets. The GCN model performs poorly on \texttt{FashionMNIST}, but is comparable to RotaTouille on the \texttt{ModelNet} dataset.

\begin{table}[ht]
    \small
	\centering
	\caption{Comparison of accuracies on the \texttt{FashionMNIST} and \texttt{ModelNet} contour datasets between RotaTouille and the baseline models. Accuracies are computed on the test dataset.}
	\label{tab:fashion_mnist_results}
	\begin{tabular}{lll}
		\toprule
		\textbf{Model} & \multicolumn{2}{c}{\textbf{Accuracy}} \\
			       & \texttt{FashionMNIST} & \texttt{ModelNet} \\ 
		\midrule
		CNN (filled contours) & $0.849\pm0.002$ & $0.898\pm0.032$ \\ 
		CNN (unfilled contours) & $0.852\pm0.001$ & $0.905\pm0.012$\\ 
        GCN & $0.626\pm0.003$ & $0.923\pm0.027$\\
        ContourCNN & $0.771\pm0.004$ & $0.849\pm0.022$\\
		\midrule
		RotaTouille & $0.867\pm0.002$ & $0.934\pm0.016$ \\ 
		\bottomrule
	\end{tabular}
\end{table}

\looseness=-1 The \texttt{RotatedMNIST} dataset was originally designed to evaluate model robustness to rotations and has been widely used as a benchmark in prior work. We evaluate RotaTouille on this dataset to compare with existing methods. We report test error, defined as $1-\text{accuracy}$, directly quantifying misclassification and allowing comparison with prior works using the Rotated MNIST dataset. As shown in \cref{tab:rotated_mnist_results}, using only contours yields a test error of $5.70\%$, which is not particularly competitive. However, augmenting our method with the simple radial histogram (RH) feature reduces the error to $3.72\%$, outperforming some of the other methods, including a image-based CNN used in \cite{cohen2016group}. While this does not match the accuracy of the most competitive methods, it demonstrates that a contour-based representation, when combined with a basic invariant feature, can achieve competitive performance on a challenging rotation benchmark.

\begin{table}[ht]
    \small
	\centering
	\caption{Comparison of test errors (in percent) on the \texttt{RotatedMNIST} dataset between existing methods and RotaTouille, both with and without the invariant radial histogram (RH) feature. Results marked with * are taken from previous papers (not re-implemented).}
	\label{tab:rotated_mnist_results}
	\begin{tabular}{ll}
		\toprule
		\textbf{Method} & \textbf{Test error ($\%$)} \\
		\midrule
		SVM (RBF kernel)* \cite{larochelle2007empirical} & $10.38\pm0.27$\\ 
		TIRBM* \cite{sohn2012learning} & $4.2$\\
		RC-RBM+Gradients IHOF* \cite{schmidt2012learning} & $3.98$\\
		CNN* \cite{cohen2016group} & $5.03\pm0.0020$ \\
		P4CNN* \cite{cohen2016group} & $2.28\pm0.0004$\\
		H-Net* \cite{worrall2017harmonic} & $1.69$\\
        GCN & $48.44\pm 0.79$\\
        ContourCNN \cite{droby2021contourcnn} & $21.79\pm 1.12$\\
		\midrule
		RotaTouille (contours only) & $5.70\pm0.13$ \\ 
		RotaTouille (contours + RH feature) & $3.72\pm0.11$ \\ 
		\bottomrule
	\end{tabular}
\end{table}

\paragraph{Shape reconstruction} \looseness=-1 To demonstrate the flexibility of RotaTouille, we include an unsupervised learning task. Here, we train an autoencoder for contour reconstruction on the \texttt{PCST} dataset, and compare it to a similar architecture on binary images. We use the mean square error (MSE) as the reconstruction loss function. Using $10\%$ of the dataset, we evaluate performance by visual inspection and compute the intersection over union (IoU) between the original shape and its reconstruction. For the contour-based model, we first convert the contours to binary images to compute the IoU scores.

\looseness=-1 See \cref{fig:autoencoder_reconstructions} for visualizations of the reconstructions. Both models are able to reconstruct the overall shapes, but seem to struggle with high frequency details. Based on the appearance of the reconstructions, our method is able to capture sharp corners better than the image-based model. This is particularly clear in example 2 and 3 of \cref{fig:autoencoder_reconstructions}. Another advantage of the contour-based model is that it is guaranteed to produce valid contours, while the image-based model can produce invalid shapes with holes or multiple components. Moreover, it is not restricted to a fixed pixel grid. In terms of IoU, both models scores $0.97$ on the validation data. 

\begin{figure}[ht]
	\centering
	\raisebox{1.2em}{ 0. }\includegraphics[width=0.3\textwidth]{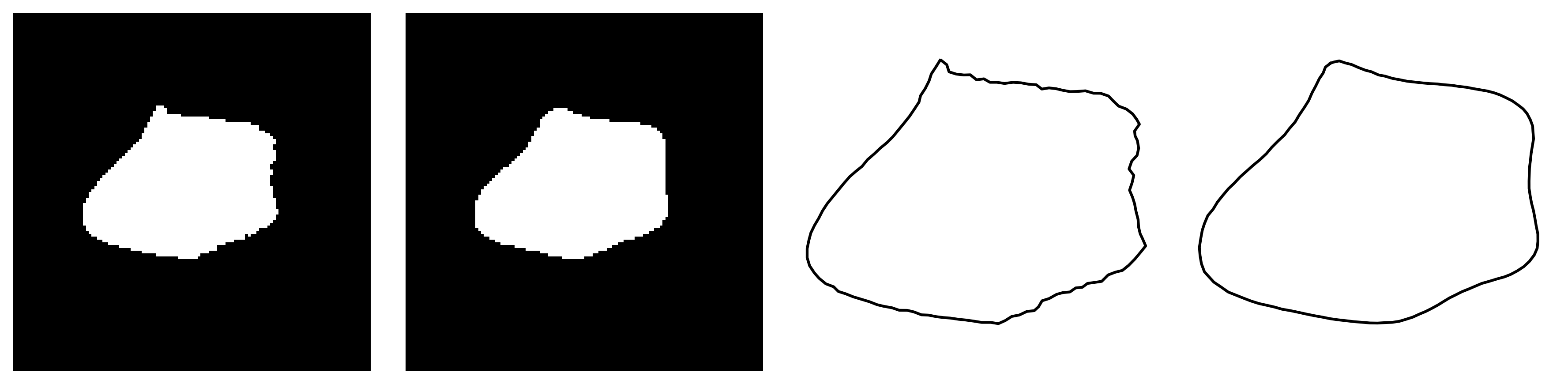}\qquad
	\raisebox{1.2em}{ 1. }\includegraphics[width=0.3\textwidth]{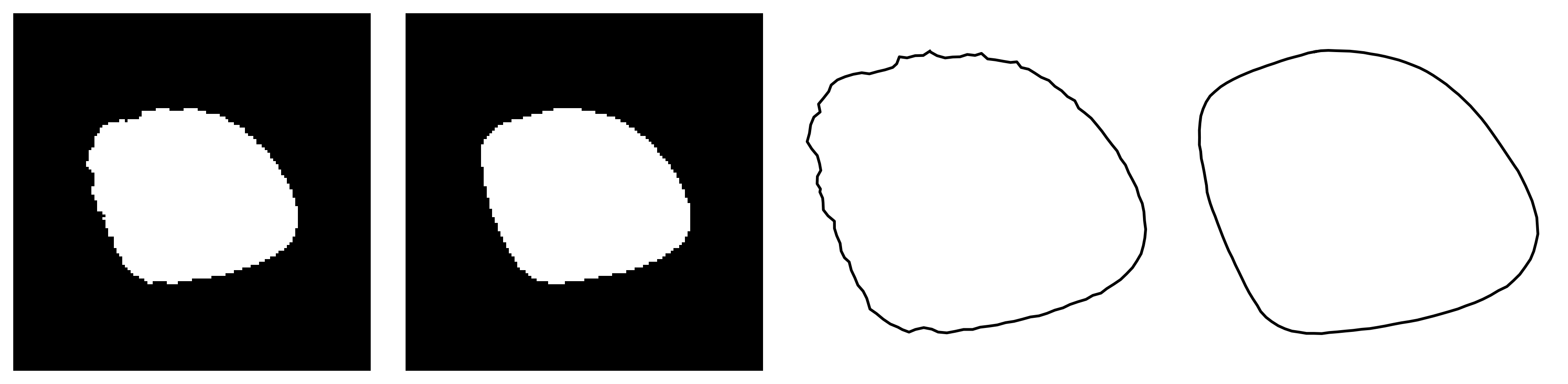}

	\raisebox{1.2em}{ 2. }\includegraphics[width=0.3\textwidth]{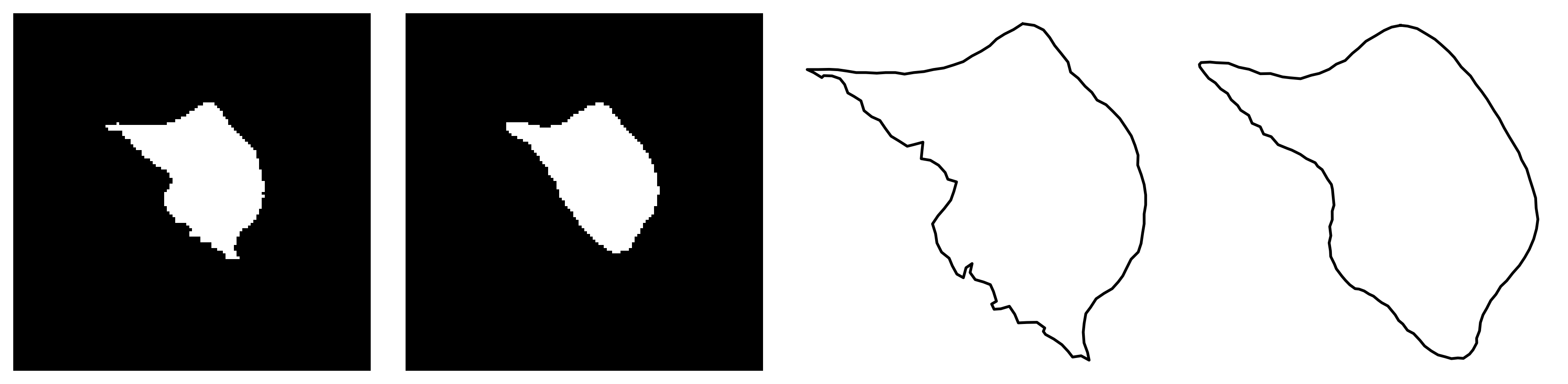}\qquad
	\raisebox{1.2em}{ 3. }\includegraphics[width=0.3\textwidth]{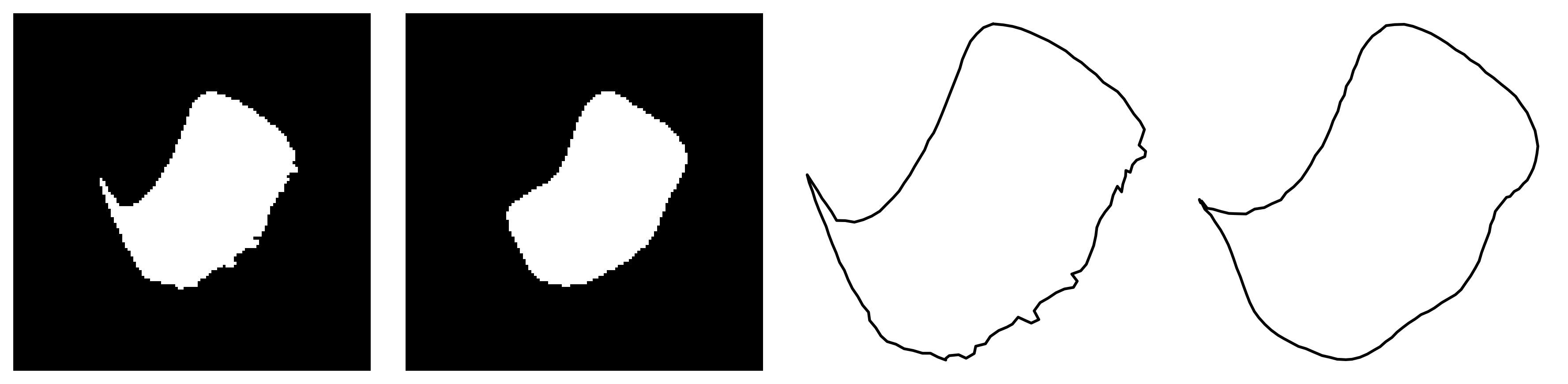}
	
	\raisebox{1.2em}{ 4. }\includegraphics[width=0.3\textwidth]{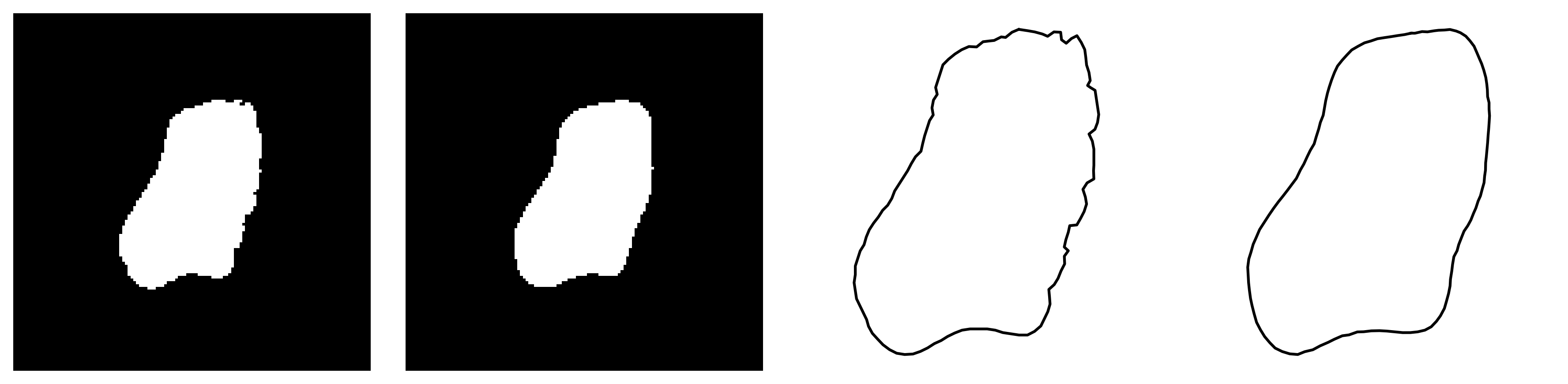}\qquad
	\raisebox{1.2em}{ 5. }\includegraphics[width=0.3\textwidth]{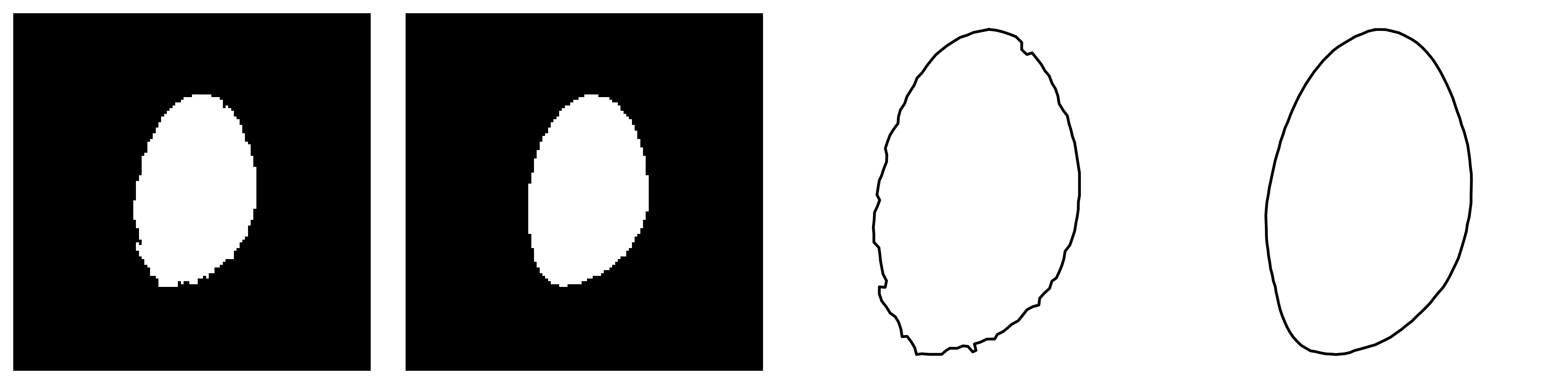}

	\raisebox{1.2em}{ 6. }\includegraphics[width=0.3\textwidth]{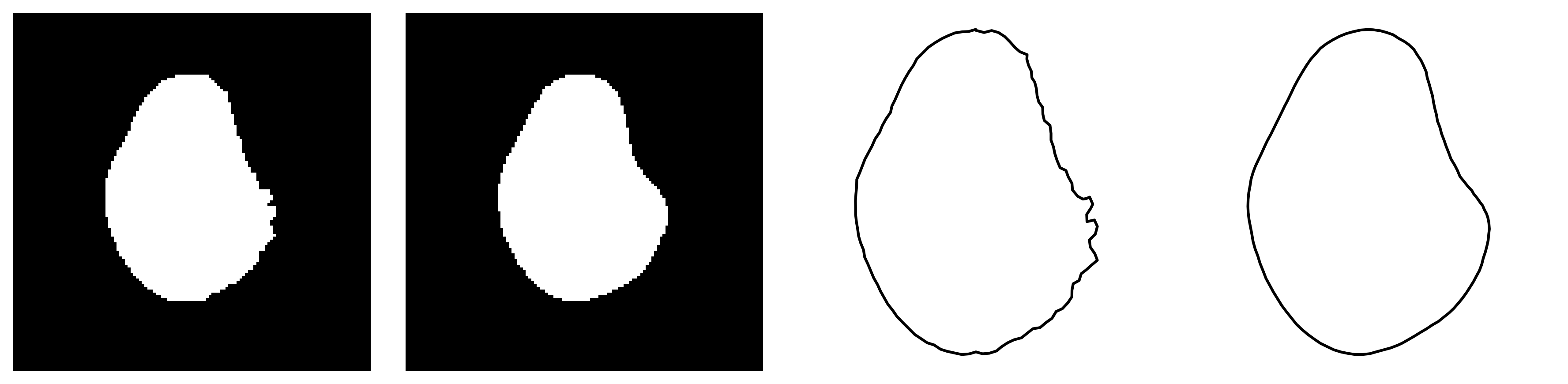}\qquad
	\raisebox{1.2em}{ 7. }\includegraphics[width=0.3\textwidth]{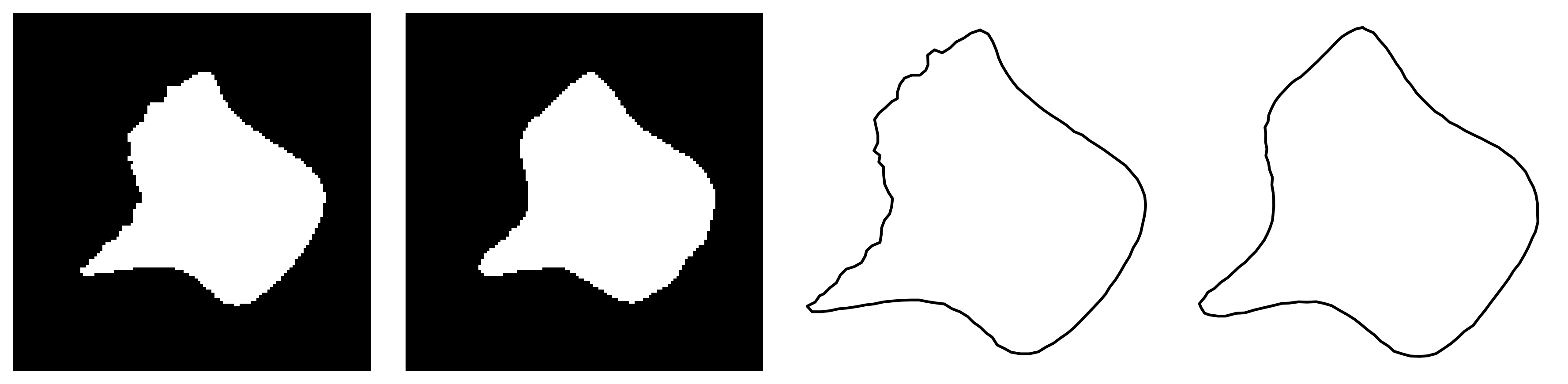}

	\raisebox{1.2em}{ 8. }\includegraphics[width=0.3\textwidth]{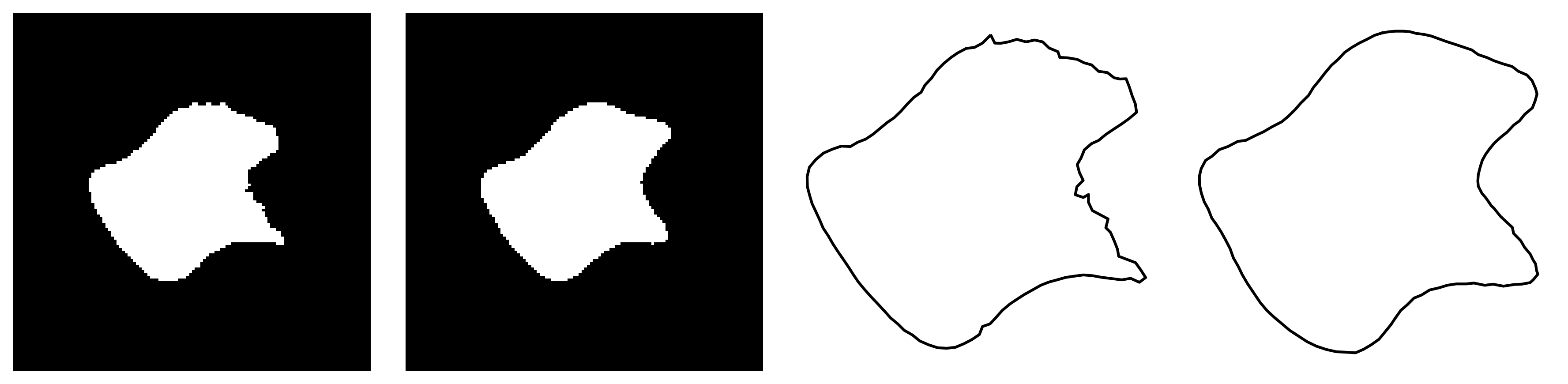}\qquad
	\raisebox{1.2em}{ 9. }\includegraphics[width=0.3\textwidth]{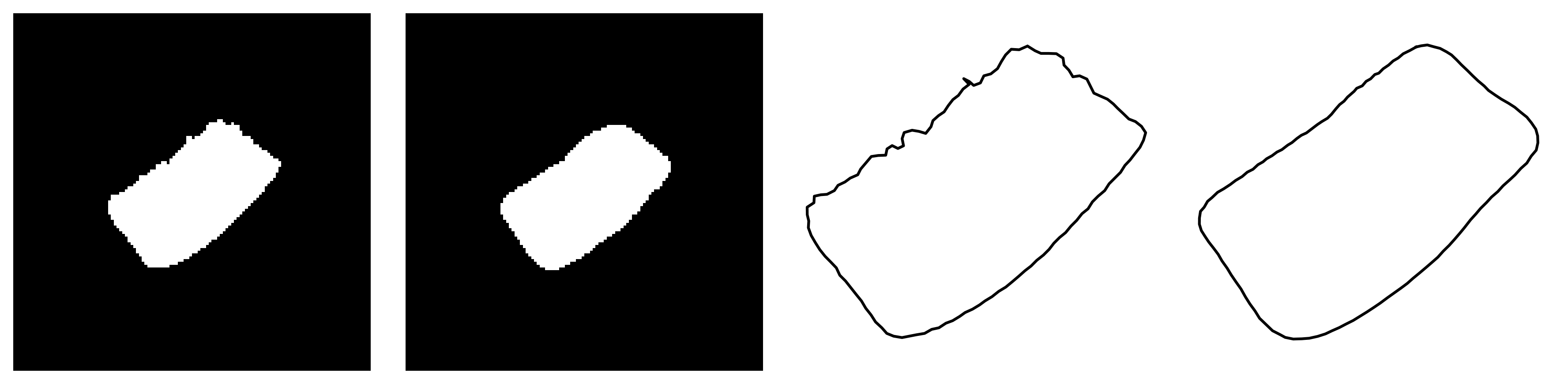}
	\caption{Ten randomly selected \texttt{PCST} validation examples (two per row). From left to right: original image, image-based autoencoder reconstruction, original contour, contour-based autoencoder reconstruction.}
	\label{fig:autoencoder_reconstructions}
\end{figure}

\paragraph{Curvature regression} We also consider a node-level regression task in which the goal is to predict the curvature at each point of a discrete contour from the \texttt{Curvature} dataset. For this task, we use the mean absolute error (MAE) as both the loss function and the evaluation metric. As baselines, we implement a 1D real-valued CNN and a graph-based GCN model operating on the Cartesian coordinates of the contour points, as well as a circle fitting method that estimates curvature by fitting a circle to every three consecutive contour points and computing the curvature from the fitted radius. The baseline CNN employs circular convolutions like ContourCNN, but without priority pooling. As shown in \cref{tab:curvature_regression_results}, RotaTouille achieves the lowest MAE and highest $R^2$.


\begin{table}[ht]
    \small
	\centering
	\caption{Mean absolute errors (MAE) and $R^2$ scores on the test set for the curvature regression task.}
	\label{tab:curvature_regression_results}
	\begin{tabular}{lll}
		\toprule
		\textbf{Model} & \textbf{MAE} & $\phantom{-}\mathbf{R^2}$ \\
		\midrule
		Circle fitting & $0.4411$ & $\phantom{-}0.1071$ \\
		CNN (real-valued) & $0.4479\pm0.0046$ & $\phantom{-}0.2220\pm0.0077$\\ 
        GCN & $0.9063\pm0.0003$ & $-0.0064\pm0.0001$ \\
		\midrule
		RotaTouille & $0.3944\pm0.0092$ & $\phantom{-}0.2480\pm0.0149$\\
		\bottomrule
	\end{tabular}
\end{table}

\subsection{Ablation Study}\label{sec:ablation_study}

We investigated how different design choices in our framework affect model performance. All ablation studies were conducted on the \texttt{FashionMNIST} contour dataset, with all other components and the architecture kept fixed as in the main experiments. We evaluate several configurable components of the model. Local pooling (coarsening) is tested with no pooling, or with strided and coset pooling using mean, max, or a learnable aggregation. We test the activation functions listed in \cref{tab:equivariant_activations}. The sampling rate of contour points is varied across 16, 32, 64, 128, and 256, while kernel sizes range from 3 to 13 in odd increments. For global pooling, we compare mean, max, and a learnable combination as in \cref{eq:global_pooling_learnable}. Numerical results are provided in \cref{sec:ablation_results} in the appendix.

We observe minimal performance differences across configurations, except for the local pooling strategy: Despite its approximate equivariance, strided pooling consistently outperforms coset pooling, while omitting pooling yields similar accuracy but increases training time due to more contour points. Performance is largely robust to kernel size, while sampling rate has a small positive effect, with mean accuracy increasing from $0.855$ (16 points) to $0.878$ (256 points). This robustness likely reflects the simple shapes in \texttt{FashionMNIST}. For more complex shapes, appropriate choice of sampling rate and kernel sizes may be needed to capture high-frequency details.

\section{Conclusion}

We introduced a framework for deep learning on contours that is equivariant to rotations by defining an action of the group $\Z_n\times S^1$ on contours and constructing corresponding complex-valued convolutional layers. We also developed non-linear activation functions and pooling operations that preserve equivariance, as well as a global pooling layer to produce invariant features. While the performance gains are modest, RotaTouille provides a easy-to-implement framework that explicitly encodes rotational symmetry on contour data. This makes it a promising option for practical applications in shape analysis and other fields where rotation equivariance and invariance is essential.

\printbibliography

\newpage
\appendix
\section{Appendix}

\subsection{Groups and Group Actions}\label{sec:groups_and_actions}

We recall some basic definition from group theory used in this manuscript. A \emph{group} $G$ is a non-empty set together with a binary operation $G\times G\to G$ mapping $(g,h)\mapsto gh$ and a distinguished element $e\in G$ called the \emph{identity element} of $G$ such that the following identities hold:
\begin{description}
    \item[Associativity.] For all $g,h,k\in G$, we have $(gh)k=g(hk)$,
    \item[Identity.] For all $g\in G$, we have $eg=g=ge$, and 
    \item[Inverses.] For all $g\in G$, there exists an $h\in G$ such that $gh=e=hg$. The element $h$ is unique, and we usually denote it by $g^{-1}$.
\end{description}

Given two groups $G$ and $H$, one can form the \emph{product group} $G\times H$ consisting of all pairs $(g,h)$ with $g\in G$ and $h\in H$. The binary operation is defined coordinate-wise as $(g_1,h_1)(g_2,h_2)=(g_1g_2,h_1h_2)$ for $g_1,g_2\in G$ and $h_1,h_2\in H$.

For a group $G$ with identity element $e$, and a set $S$, a \emph{(left) group action} of $G$ on $S$ is a map $G\times S\to S$ mapping $(g,x)\mapsto g\bigcdot x$ that satisfies the following two conditions:
\begin{description}
    \item[Identity.] For all $x\in S$, we have $e\bigcdot x=x$, and
    \item[Compatibility.] For all $g,h\in G$ and $x\in S$, we have $g\bigcdot(h\bigcdot x)=(gh)\bigcdot x$.
\end{description}

\begin{proposition}\label{prop:group_action_well_defined}
Let $G_n=\Z_n\times S^1$ and $X_n^k=\map(\Z_n,\C^k)$. The map $G_n\times X_n^k\to X_n^k$ mapping $((l,w), x)\mapsto (l,w)\bigcdot x$ where $((l,w)\bigcdot x)(q)=wx(q-l)$ for $q\in\Z_n$ is a group action.
\end{proposition}
\begin{proof}
Let $x\in X_n^k$. The identity element in $G_n$ is the element $(0,1)$, so the map $e\bigcdot x$ sends every $q\in Z_n$ to $1x(q-0)=x(q)$. For the compatibility condition, we let $(l,w),(l',w')\in G_n$, $q\in\Z_n$ and straight-forward computation shows that
\begin{align*}
(((l,w)(l',w'))\bigcdot x)(q)
&=((l+l',ww')\bigcdot x)(q)
=ww'x(q-(l+l'))\\
&=w'wx(q-l-l')
=((l',w')\bigcdot((l,w)\bigcdot x))(q).
\end{align*}
Since this holds for all $q\in\Z_n$, the maps $((l,w)(l',w'))\bigcdot x$ and $(l',w')\bigcdot((l,w)\bigcdot x)$ are the same.
\end{proof}

\subsection{Proofs} \label{sec:proofs}

\begin{proposition}\label{prop:linear_shift_eq_map_is_conv}
If $T\colon X_n^1\to X_n^1$ is $\C$-linear and commutes with cyclic shifts, then $T$ is a circular convolution operator.
\end{proposition}

\begin{proof}    
Let $\delta_j\colon[n]\to\C$ be the Kronecker delta function defined by
\begin{equation*}
\delta_j(q) = \begin{cases}
	1 & \text{if } q=j\text{, and}\\
	0 & \text{otherwise.}
\end{cases}
\end{equation*}

Every $x\in X_n^1$ can be written uniquely as $x=\sum_{j=0}^{n-1} x(j)\delta_j$. If we write $S\colon X_n^1\to X_n^1$ for the cyclic shift operator defined by $S(x)(q)=x(q-1)$, then we can express the cyclic shift equivariance as $T\circ S=S\circ T$. On basis elements, we have $S(\delta_j)=\delta_{j+1}$ so $\delta_j=S^j(\delta_0)$ for all $j\in\Z_n$. Thus, we can rewrite $x=\sum_{j=0}^{n-1}x(j)S^j(\delta_0)$ and since $T$ is linear and commutes with $S$, we have that 
\begin{align*}
	T(x)=T\left(\sum_{j=0}^{n-1}x(j)S^j(\delta_0)\right) &= \sum_{j=0}^{n-1}x(j)T(S^j(\delta_0)) = \sum_{j=0}^{n-1}x(j)S^j(T(\delta_0))\\
\end{align*}
and hence $T(x)(q)=\sum_{j=0}^{n-1}x(j)y(q-j)$ where $y=T(\delta_0)\in X_n^1$. This is precisely the circular convolution of $x$ and $y$.
\end{proof}

\convequivariant*
\begin{proof}
For $(l,w)\in G_n$, $q\in\Z_n$ and $\phi\in\Phi$, we have 
\begin{align*}
	((l,w)\bigcdot\conv_{\phi}(x))(q) &= w\sum_{j=1}^k(\phi_j\ast x_j)(q-l) = \sum_{j=1}^k w(\phi_j\ast x_j)(q-l) \\
					  &= \sum_{j=1}^k(\phi_j\ast ((l,w)\bigcdot x_j))(q)= \conv_{\phi}((l,w)\bigcdot x)(q)
\end{align*}
for any $x\in X_n^k$. Since $\cconv_\Phi$ and the group action are defined coordinate-wise, we are done.
\end{proof}

\nonlinearequivariant*
\begin{proof}
If $a(z) = g(\vert z\vert)z$, then $a(wz)=g(\vert wz\vert)wz=wg(|z|)z=wa(z)$ since $|w|=1$. Now, assume that $a$ satisfies the equivariance condition. We have $a(0)=wa(0)$ for all $w\in S^1$, so $a(0)$ must be zero. If $z\neq 0$, write $z=rw$ with $r>0$ and $w\in S^1$, and observe that $a(z) = a(rw) = wa(r) =  zr^{-1}a(r).$ Define the function $g$ by letting $g(0)=0$ and $g(r)=r^{-1}a(r)$ whenever $r\neq0$.
\end{proof}

\subsection{Training Details}

\begin{table}[H]
	\centering
    \caption{Hyperparameters for the training of the models. All models were trained using the Adam optimizer with the specified learning rate. Hyperparameter values were chosen based on performance on validation data. The number of (real) parameters is an estimate of the model complexity. Random rotations were applied to improve generalization for the non-equivariant and non-invariant models.} 
    \label{tab:training_details}
	\tabcolsep=0.11cm
	\resizebox{\textwidth}{!}{
	\begin{tabular}{lllccccc}
	\toprule
	\textbf{Model} & \textbf{Dataset} & \textbf{LR} & \textbf{Batch Size} & \textbf{Epochs} & \textbf{Parameters} & \textbf{Data Aug.} \\
	\midrule
	CNN (2D) & \texttt{FashionMNIST} & $0.01$ & $128$ & $200$ & $294\,666$ & Yes \\
	GCN & \texttt{FashionMNIST} & $0.001$ & $128$ & $200$ & $74\,826$ & Yes \\
    ContourCNN & \texttt{FashionMNIST} & $0.001$ & $128$ & $200$ & $42\,874$ & Yes \\
	RotaTouille & \texttt{FashionMNIST} & $0.0005$ & $128$ & $200$ & $65\,089$ & No \\
	CNN (2D) & \texttt{ModelNet} & $0.01$ & $16$ & $200$ & $1\,143\,012$ & Yes \\
	GCN & \texttt{ModelNet} & $0.001$ & $16$ & $200$ & $75\,594$ & Yes \\
    ContourCNN & \texttt{ModelNet} & $0.001$ & $16$ & $200$ & $42\,964$ & Yes \\
	RotaTouille & \texttt{ModelNet} & $0.0005$ & $16$ & $200$ & $64\,747$ & No \\
    GCN & \texttt{RotatedMNIST} & $0.001$ & $128$ & $200$ & $74\,826$ & Yes \\
    ContourCNN & \texttt{RotatedMNIST} & $0.001$ & $128$ & $200$ & $42\,874$ & Yes \\
	RotaTouille & \texttt{RotatedMNIST} & $0.0005$ & $128$ & $200$ & $65\,089$ & No \\
	RotaTouille + RH & \texttt{RotatedMNIST} & $0.0005$ & $128$ & $200$ & $66\,881$ & No \\
	CNN (2D) &  \texttt{PCST} & $0.001$ & $32$ & $200$ & $5\,315$ & Yes \\
	RotaTouille &  \texttt{PCST} & $0.001$ & $32$ & $200$ & $4\,378$ & No \\
	CNN (1D) &  \texttt{Curvature} & $0.001$ & $32$ & $100$ & $34\,033$ & Yes \\
    GCN &  \texttt{Curvature} & $0.001$ & $32$ & $100$ & $41\,473$ & Yes \\
	RotaTouille &  \texttt{Curvature} & $0.001$ & $32$ & $100$ & $27\,269$ & No \\
	\bottomrule
\end{tabular}
}
\end{table}

\subsection{Model Details}\label{sec:model_details}

Here, we provide detailed descriptions of the model architectures used in the experiments.

\subsubsection{RotaTouille Models}

\paragraph{Classification model} The \emph{ConvBlock} used in the classifier consists of a sequence of operations that together form the basic building block of the model. The structure can be summarized schematically as
\[
\text{ConvLayer} \to \text{ModReLU} \to \text{BatchNorm} \to \text{(Coarsening)} \to \text{Global pooling}.
\]
The block begins with a \emph{ConvLayer}, which performs a complex-valued circular convolution. This is followed by the equivariant activation function \emph{ModReLU}. Next, \emph{BatchNorm} is applied, operating only on the magnitudes (absolute values) of the complex activations to stabilize training. If the coarsening factor $p>1$, a \emph{Coarsening} step is applied, performing local pooling through a learnable combination of the mean and absolute-value maximum. The resulting contour is then forwarded to the next ConvBlock. Finally, a \emph{Global pooling} layer produces real-valued invariant features using a learnable combination of the mean and absolute-value maximum, which are subsequently concatenated and used by the classification head. The full classifier architecture is listed in \cref{tab:classifier_architecture}.

\begin{table}[ht]
\caption{RotaTouille classifier architecture used for the \texttt{FashionMNIST}, \texttt{ModelNet} and \texttt{RotatedMNIST} datasets. For the \texttt{ModelNet} dataset, the number of input channels in the first layer is $4$ instead of $1$.}
\label{tab:classifier_architecture}
\small
\centering
\begin{tabular}{llllp{6.5cm}}
\toprule
\textbf{Layer} & \textbf{Input} $\to$ \textbf{Output} & \textbf{Kernel} & \textbf{$p$} & \textbf{Notes} \\
\midrule
\multicolumn{5}{c}{\textbf{Feature extractor}} \\
\midrule
ConvBlock & 1 $\to$ 8 & 9 & 1 & \\
ConvBlock & 8 $\to$ 8 & 9 & 2 & \\
ConvBlock & 8 $\to$ 16 & 9 & 1 & \\
ConvBlock & 16 $\to$ 16 & 9 & 2 & \\
ConvBlock & 16 $\to$ 35 & 9 & 1 & \\
ConvBlock & 35 $\to$ 35 & 9 & 1 & \\
ConvBlock & 35 $\to$ 10 & 9 & 1 & \\
\midrule
\multicolumn{5}{c}{\textbf{Classifier head}} \\
\midrule
Linear & $m_1+m_2$ $\to$ 128 & - & - & Followed by BatchNorm, Dropout (0.5) and ReLU. The number of invariant features from the global pooling layers is denoted by $m_1$, and $m_2$ is any additional invariant features such as radial pixel intensity histograms.\\
Linear & 128 $\to n$ & - & - & Where $n$ is the number of classes.\\
\bottomrule
\end{tabular}
\end{table}

\paragraph{Contour autoencoder model} The contour autoencoder used with the \texttt{PCST} dataset consists of an encoder and a decoder part. The \emph{EncoderBlock} used in the autoencoder consists of a sequence of operations that progressively transform and coarsen the input representation. The structure can be summarized schematically as
\[
\bigl(\text{ConvLayer} \to \text{BatchNorm} \to \text{Amplitude-Phase}\bigr) \times 2 \to \text{Coarsening}.
\]
Each block consists of a \emph{ConvLayer} performing complex-valued circular convolution, followed by \emph{BatchNorm} on magnitudes and an \emph{Amplitude-Phase} activation function as an equivariant nonlinearity. This sequence is repeated twice. Finally, \emph{Coarsening} is applied using strided convolution with stride $p$ and kernel size $p$. If the number of input and output channels match, a skip connection adds the input contour to the output. The \emph{DecoderBlock} is identical to the encoder blocks, except the order is reversed, batch normalization is omitted and coarsening is replaced with strided transposed convolutions for unpooling. The full autoencoder architecture is described in \cref{tab:full_shape_autoencoder}.

\begin{table}[ht]
\caption{RotaTouille contour autoencoder architecture.}
\label{tab:full_shape_autoencoder}
\small
\centering
\begin{tabular}{llll}
\toprule
\textbf{Layer} & \textbf{Input} $\to$ \textbf{Output} & \textbf{Kernel} & \textbf{$p$} \\
\midrule
\multicolumn{4}{c}{\textbf{Encoder}} \\
\midrule
EncoderBlock & 1 $\to$ 4 & 11 & 2 \\
EncoderBlock & 4 $\to$ 4 & 9 & 2 \\
EncoderBlock & 4 $\to$ 4 & 7 & 2 \\
EncoderBlock & 4 $\to$ 4 & 5 & 2 \\
EncoderBlock & 4 $\to$ 4 & 3 & 2 \\
\midrule
\multicolumn{4}{c}{\textbf{Decoder}} \\
\midrule
DecoderBlock & 4 $\to$ 4 & 3 & 2 \\
DecoderBlock & 4 $\to$ 4 & 5 & 2 \\
DecoderBlock & 4 $\to$ 4 & 7 & 2 \\
DecoderBlock & 4 $\to$ 4 & 9 & 2 \\
DecoderBlock & 4 $\to$ 1 & 11 & 2 \\
\bottomrule
\end{tabular}
\end{table}

\paragraph{Regression model} For the node-level regression model, we do not include any local pooling layers since we are predicting one curvature value per point in the input contour. We take absolute values before the regression head to obtain rotation invariant features, but still maintaining cyclic shift equivariance. The regression model used with the \texttt{Curvature} dataset is described in \cref{tab:regression_architecture}.

\begin{table}[ht]
\caption{RotaTouille regression model architecture.}
\label{tab:regression_architecture}
\small
\centering
\begin{tabular}{lllp{6.5cm}}
\toprule
\textbf{Layer} & \textbf{Input} $\to$ \textbf{Output} & \textbf{Kernel} & \textbf{Notes} \\
\midrule
\multicolumn{4}{c}{\textbf{Equivariant feature extractor (complex-valued)}} \\
\midrule
ConvLayer & 1 $\to$ 8 & 5 & Followed by ModReLU and BatchNorm. \\
ConvLayer & 8 $\to$ 16 & 5 & Followed by ModReLU and BatchNorm. \\
ConvLayer & 16 $\to$ 32 & 5 & Followed by ModReLU and BatchNorm. \\
ConvLayer & 32 $\to$ 64 & 5 & Followed by ModReLU and BatchNorm. \\
\midrule
\multicolumn{4}{c}{\textbf{Regression head (real-valued)}} \\
\midrule
Linear & 64 $\to$ 1 & - & Element-wise absolute values as input.\\
\bottomrule
\end{tabular}
\end{table}

\FloatBarrier

\subsubsection{Baseline Models}

\paragraph{2D CNN classifier for images}  
The baseline 2D CNN classifier uses 2D convolutions with batch normalization, ReLU activations, and max pooling, followed by a fully connected classifier. The architecture is summarized in \cref{tab:cnn2d}.

\begin{table}[ht]
\caption{Baseline 2D CNN classifier for images. For the \texttt{ModelNet} dataset, the number of input channels in the first layer is $4$ instead of $1$.}
\label{tab:cnn2d}
\small
\centering
\begin{tabular}{llllp{6.5cm}}
\toprule
\textbf{Layer} & \textbf{Input} $\to$ \textbf{Output} & \textbf{Kernel} & \textbf{Pool} & \textbf{Notes} \\
\midrule
\multicolumn{5}{c}{\textbf{Feature extractor}} \\
\midrule
Conv2d & 1 $\to$ 32 & 3 & - & Followed by BatchNorm and ReLU. \\
Conv2d & 32 $\to$ 64 & 3 & - & Followed by BatchNorm and ReLU. \\
MaxPool2d & - & - &2 & \\
Conv2d & 64 $\to$ 64 & 3 & - &  Followed by BatchNorm and ReLU. \\
Conv2d & 64 $\to$ 64 & 3 & - & Followed by BatchNorm and ReLU. \\
MaxPool2d & - & - & 2 &  \\
\midrule
\multicolumn{5}{c}{\textbf{Classifier head}} \\
\midrule
Linear & $m$ $\to$ 64 & - & - & Followed by BatchNorm, Dropout (0.5) and ReLU. The embedding size $m$ depends on the input size.\\
Linear & 64 $\to$ $n$ & - & - & Where $n$ is the number of classes.\\
\bottomrule
\end{tabular}
\end{table}

\paragraph{2D CNN autoencoder} The 2D CNN autoencoder is composed of convolutional downsampling blocks, called \emph{ConvBlock2d}, which consist of two consecutive sequences of convolution, batch normalization, and ReLU activation, followed by a downsampling operation. The upsampling counterpart, \emph{DeconvBlock2d}, consists of an upsampling operation followed by two consecutive sequences of ReLU activation and convolution. We use strided (transposed) convolutions for the upsampling and downsampling operations. The full architecture is summarized in \cref{tab:autoencoder2d}.

\begin{table}[h]
\caption{Baseline 2D CNN autoencoder.}
\label{tab:autoencoder2d}
\small
\centering
\begin{tabular}{llll}
\toprule
\textbf{Layer} & \textbf{Input} $\to$ \textbf{Output} & \textbf{Kernel} & \textbf{Pool / Upsample}\\
\midrule
\multicolumn{4}{c}{\textbf{Encoder}} \\
\midrule
ConvBlock2d & 1 $\to$ 4 & 3 & 4\\
ConvBlock2d & 4 $\to$ 4 & 3 & 4 \\
ConvBlock2d & 4 $\to$ 8 & 3 & 4 \\
\midrule
\multicolumn{4}{c}{\textbf{Decoder}} \\
\midrule
DeconvBlock2d & 8 $\to$ 4 & 3 & 4 \\
DeconvBlock2d & 4 $\to$ 4 & 3 & 4 \\
DeconvBlock2d & 4 $\to$ 1 & 3 & 4 \\
\bottomrule
\end{tabular}
\end{table}

\paragraph{Graph convolutional network (GCN)} The GCN classification and regression models are both based on GCN layers \cite{kipf2016semi} for cyclic graphs with self-loops. The full architecture is summarized in \cref{tab:cyclegcn_summary}. For regression on the \texttt{Curvature} dataset, we use $n_{\text{layers}} = 3$ and $n_{\text{layers}} = 5$ for classification tasks.

\begin{table}[ht]
\caption{Baseline GCN architectures for contour classification and node-level regression. For the \texttt{ModelNet} dataset, the number of input channels in the first layer is $8$ instead of $2$.}
\label{tab:cyclegcn_summary}
\small
\centering
\begin{tabular}{llp{6.5cm}}
\toprule
\textbf{Layer} & \textbf{Input} $\to$ \textbf{Output} & \textbf{Notes} \\
\midrule
\multicolumn{3}{c}{\textbf{Feature extractor}} \\
\midrule
GCNLayer & $2 \to 128$ & Repeated $n_{\text{layers}}$ times with ReLU activation.\\
\midrule
\multicolumn{3}{c}{\textbf{Classifier head}} \\
\midrule
GlobalMeanPool & $128 \to 128$ & Global average pooling across contour points. \\
Linear & $128 \to 64$ & Followed by BatchNorm, Dropout (0.5) and ReLU.\\
Linear & $64 \to n$ & Where $n$ is the number of classes. \\
\midrule
\multicolumn{3}{c}{\textbf{Regression head}} \\
\midrule
Linear & $128 \to 64$ & Followed by BatchNorm, Dropout (0.5) and ReLU.\\
Linear & $64 \to 1$ & \\
\bottomrule
\end{tabular}
\end{table}

\paragraph{ContourCNN and 1D CNN} The ContourCNN is used for contour classification, while the 1D CNN is used for node-level regression. Each \emph{CircConvBlock} in the ContourCNN consists of a circular 1D convolution, batch normalization, ReLU activation, and priority pooling as in \cite{droby2021contourcnn}. Both networks operate on real-valued contour coordinates.

\begin{table}[ht]
\caption{Baseline ContourCNN classifier architecture used for contour classification. For the \texttt{ModelNet} dataset, the number of input channels in the first layer is $8$ instead of $2$.}
\label{tab:contourcnn_architecture}
\small
\centering
\begin{tabular}{lllp{6.5cm}}
\toprule
\textbf{Layer} & \textbf{Input} $\to$ \textbf{Output} & \textbf{Kernel} & \textbf{Notes} \\
\midrule
\multicolumn{4}{c}{\textbf{Feature extractor}} \\
\midrule
CircConvBlock & 2 $\to$ 32 & 3 & Priority pooling to $l=40$. \\
CircConvBlock & 32 $\to$ 64 & 3 & Priority pooling to $l=30$. \\
CircConvBlock & 64 $\to$ 128 & 3 & Priority pooling to $l=20$. \\
\midrule
\multicolumn{4}{c}{\textbf{Classifier head}} \\
\midrule
GlobalMeanPool & 128 $\to$ 128 & - & Global average pooling across contour points. \\
Linear & 128 $\to$ 80 & - & Followed by BatchNorm, Dropout (0.5) and ReLU.\\
Linear & 80 $\to n$ & - & Where $n$ is the number of classes. \\
\bottomrule
\end{tabular}
\end{table}

\begin{table}[ht]
\caption{Baseline 1D CNN regressor architecture used for node-level regression.}
\label{tab:cnn1d_regression_architecture}
\small
\centering
\begin{tabular}{lllp{6.5cm}}
\toprule
\textbf{Layer} & \textbf{Input} $\to$ \textbf{Output} & \textbf{Kernel} & \textbf{Notes} \\
\midrule
\multicolumn{4}{c}{\textbf{Feature extractor}} \\
\midrule
Conv1d & 2 $\to$ 16 & 5 & Followed by BatchNorm and ReLU.\\
Conv1d & 16 $\to$ 32 & 5 & Followed by BatchNorm and ReLU.\\
Conv1d & 32 $\to$ 64 & 5 & Followed by BatchNorm and ReLU.\\
Conv1d & 64 $\to$ 64 & 5 & Followed by BatchNorm and ReLU.\\
\midrule
\multicolumn{4}{c}{\textbf{Regression head}} \\
\midrule
Linear & 64 $\to$ 1 & - & \\
\bottomrule
\end{tabular}
\end{table}

\FloatBarrier

\subsection{Ablation Study Results}\label{sec:ablation_results}

Here, we list all numerical results for the ablation study in \cref{sec:ablation_study}.

\begin{table}[ht]
    \small
	\centering
	\caption{Comparison of test accuracies on the \texttt{FashionMNIST} dataset for different choices of \textbf{local pooling (coarsening) functions} and aggregation functions. Note that in some training runs, coset pooling led to numerical instabilities (NaN losses), preventing convergence. For consistency, we report mean accuracy scores only over the successfully converged.}
	\label{tab:ablation_coarsening_functions}
	\begin{tabular}{lll}
		\toprule
		\textbf{Pooling type} & \textbf{Aggregation} & \textbf{Accuracy} \\
		\midrule
        No Pooling & -- & $0.869\pm0.015$\\
        Strided & Mean & $0.868\pm0.002$\\
        Coset & Mean & $0.827\pm0.025$\\

        Strided & Max & $0.865\pm0.003$\\
        Coset & Max & $0.842\pm0.003$\\

        Strided & Combined & $0.867\pm0.002$\\
        Coset & Combined & $0.841\pm0.001$\\
		\bottomrule
	\end{tabular}
\end{table}%
\begin{figure}[ht]
  \centering
  \includesvg[width=0.6\textwidth,inkscapelatex=true]{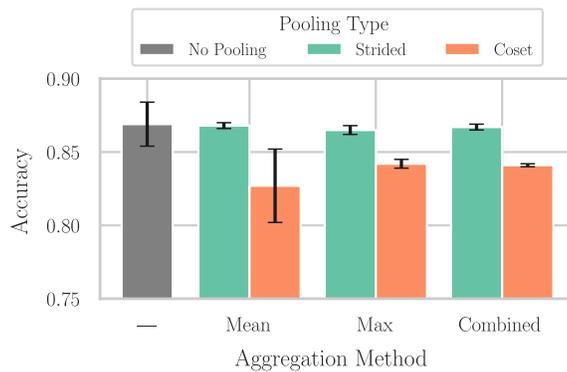}
  \caption{A bar plot visualizing the numerical results in \cref{tab:ablation_coarsening_functions}.}
  \label{fig:ablation_coarsening_plot}
\end{figure}%
\begin{table}[ht]
    \small
	\centering
	\caption{Comparison of test accuracies on the \texttt{FashionMNIST} dataset for different choices of \textbf{non-linear activation functions}. See \cref{tab:equivariant_activations} for definitions of the activation functions used.}
	\label{tab:ablation_activation_functions}
	\begin{tabular}{ll}
		\toprule
		\textbf{Activation function} & \textbf{Accuracy} \\
		\midrule
        Siglog & $0.868\pm0.002$\\
        ModReLU & $0.867\pm0.002$\\
        Amplitude-Phase ($\tanh$) & $0.869\pm0.001$\\
		\bottomrule
	\end{tabular}
\end{table}%
\begin{table}[ht]
    \small
	\centering
	\caption{Comparison of test accuracies on the \texttt{FashionMNIST} dataset for different choices of the \textbf{number of sampled contour points}.}
	\label{tab:ablation_contour_size}
	\begin{tabular}{ll}
		\toprule
		\textbf{Contour points} & \textbf{Accuracy} \\
		\midrule
        16 & $0.855\pm0.003$\\
        32 & $0.866\pm0.003$\\
        64 & $0.873\pm0.002$\\
        128 & $0.877\pm0.002$\\
        256 & $0.878\pm0.002$\\
		\bottomrule
	\end{tabular}
\end{table}%
\begin{table}[ht]
    \small
	\centering
	\caption{Comparison of test accuracies on the \texttt{FashionMNIST} dataset for different choices \textbf{kernel sizes}.}
	\label{tab:ablation_kernel_sizes}
	\begin{tabular}{lll}
		\toprule
		\textbf{Kernel size} & \textbf{Parameters} & \textbf{Accuracy} \\
		\midrule
        3 & $33\,997$ & $0.870\pm0.003$\\
        5 & $44\,361$ & $0.875\pm0.002$\\
        7 & $54\,725$ & $0.876\pm0.002$\\
        9 & $65\,089$ & $0.876\pm0.002$\\
        11 & $75\,453$ & $0.875\pm0.002$\\
        13 & $85\,817$ & $0.876\pm0.002$\\
		\bottomrule
	\end{tabular}
\end{table}%
\begin{table}[ht]
    \small
	\centering
	\caption{Comparison of test accuracies on the \texttt{FashionMNIST} dataset for different choices of the aggregation function used on the contour points' absolute values in the \textbf{global pooling} layer for producing invariant representaion.}
	\label{tab:ablation_global_pool}
	\begin{tabular}{ll}
		\toprule
		\textbf{Aggregation function} & \textbf{Accuracy} \\
		\midrule
        Mean & $0.864\pm0.003$\\
        Max & $0.869\pm0.002$\\
        Combined & $0.867\pm0.002$\\
		\bottomrule
	\end{tabular}
\end{table}

\clearpage

\end{document}